\documentclass[12pt]{article}
\usepackage[usenames]{color}
\usepackage{graphicx, subfigure}
\usepackage{amsmath, amsthm, amssymb}
\usepackage{amsfonts}
\usepackage{fullpage}
\usepackage{ifthen}
\usepackage{url}
\usepackage[sort&compress]{natbib}
\usepackage{multirow}

\usepackage[linesnumbered,algoruled,boxed,lined,commentsnumbered]{algorithm2e}

\makeatletter
\def\@biblabel#1{}
\makeatother

\theoremstyle{plain}
\newtheorem{thm}{Theorem}%[section]
\theoremstyle{definition}
\newtheorem{example}{Example}%[section]

\theoremstyle{remark}

\newcommand{\E}{\mathsf{E}}

\newcommand{\prob}{\mathsf{P}}

\newcommand{\pl}{\mathsf{pl}}

\newcommand{\nm}{{\sf N}}

\newcommand{\gam}{{\sf Gamma}}

\newcommand{\Pareto}{{\sf Pareto}}

\newcommand{\Binom}{{\sf Binom}}

\title{On some practical challenges of conformal prediction}

\author{Liang Hong$^*$\qquad Noura Raydan Nasreddine\footnote{Department of Mathematical Sciences, The University of Texas at Dallas, 800 West Campbell Road, Richardson, TX 75080, USA. All correspondence should be addressed to liang.hong@utdallas.edu.}}

\date{\today}

\begin{document}

\maketitle

\begin{abstract}
Conformal prediction is a model-free machine learning method for constructing prediction regions at a guaranteed coverage probability level.  However,  a data scientist often faces three challenges in practice: (i) the determination of a conformal prediction region is only approximate,  jeopardizing the finite-sample validity of prediction, (ii) the computation required could be prohibitively expensive,  and (iii) the shape of a conformal prediction region is hard to control.  This article offers new insights into the relationship among the monotonicity of the non-conformity measure,  the monotonicity of the plausibility function,  and the exact determination of a conformal prediction region.  Based on these new insights,  we propose a quadratic-polynomial non-conformity measure that allows a data scientist to circumvent the three challenges simultaneously within the full conformal prediction framework.

\smallskip

\end{abstract}

\section{Introduction}

\subsection{Conformal prediction and three practical challenges}

Consider the typical regression setting where the true data-generating mechanism is 
\begin{equation}
\label{eq:true}
Y=f(X)+\varepsilon,
\end{equation}
where $Y\in \mathbb{R}$  is the response variable, $X\in \mathbb{R}^p$ for some $p\geq 1$ is the vector of $p$ predictors, and $\varepsilon$ is the random error term.  It is customary to assume $\E[\varepsilon]=0$.  We do not adopt this assumption, since it is unnecessary for the method proposed in this article.  Suppose $Z_1 = (X_1, Y_1), Z_2=(X_2, Y_2), \ldots$ is a sequence of exchangeable observations from (\ref{eq:true}), where each $Z_i$ follows a distribution $\prob$.  Our goal is to perform an interval prediction of the next response $Y_{n+1}$ at a randomly sampled feature $X_{n+1}$, based on past observations of $Z^n = \{Z_1,\ldots,Z_n\}$.  If one takes a parametric model approach, there will be two potential dangers lurking behind the scene: (i) model misspecification (e.g., Claeskens and Hjort 2008) and (ii) the effect of selection (e.g.,  Leeb 2009; Berk et al. 2013; Hong et al.  2018; Kuchibhotal et al. 2022).  While these two issues can be largely avoided by employing a non-parametric model,  predictions based on a non-parametric model are usually only asymptotically valid.  These issues associated with a model-based approach prompted researchers to seek a model-free approach for creating valid prediction regions.  Early works in this direction include Wilks (1941) and Fligner and Wolfe (1976).  However, these authors only treated the unsupervised learning case, and their methods are not immediately applicable in  the regression setting.  For prediction in the regression setting,  conformal prediction (e.g.,  Vovk et al.  2005, 2009; Shafer and Vovk 2008; Barber et al. 2021) is a model-free machine learning method for generating finite-sample valid prediction regions at a given confidence level.

To apply conformal prediction,  we first choose a \emph{non-conformity measure} $M(B, z)$ which is a real-valued deterministic mapping of two arguments, where the first argument $B=\{z_1, \ldots, z_n\}$ is a \emph{bag}, i.e., a collection, of observed data and the second argument $z=(x, y)$ is a provisional value of a future observation $Z_{n+1}$.  Then we run the following Algorithm~\ref{algo:conformal.s}:

\begin{algorithm}[H]
Initialize: data $z^n = \{z_1,\ldots,z_n\}$ and $x_{n+1}$, non-conformity measure $M$, and a possible $y$ value\;
Set $z_{n+1} = (x_{n+1}, y)$ and write $z^{n+1} = z^n \cup \{z_{n+1} \}$\;
Define $\mu_i = M(z^{n+1} \setminus \{z_i\}, z_i)$ for $i=1,\ldots,n,n+1$\;
Compute $\pl_{x_{n+1}, z^n}(y) = (n+1)^{-1} \sum_{i=1}^{n+1} 1\{\mu_i \geq \mu_{n+1}\}$\;
Return $\pl_{x_{n+1}, z^n}(y)$.
\caption{\bf Conformal prediction (supervised learning)}
\label{algo:conformal.s}
\end{algorithm}

\bigskip

In Algorithm~\ref{algo:conformal.s}, $1_E$ stands for the indicator function of an event $E$. The quantity $\mu_i$,  called the $i$-th \emph{non-conformity score},  assigns a numerical score to $z_i$ to indicate how much $z_i$ agrees with the data in the bag $B=z^n\cup\{z_{n+1}\}\backslash\{z_i\}$, where $z_i$ itself is excluded to avoid biases as in leave-out-one cross-validation. Algorithm~\ref{algo:conformal.s} corresponds to the function $\pl_{x_{n+1}, z^n}$ that outputs a value between $0$ and $1$ based on all non-conformity scores.  The output of $\pl_{x_{n+1}, z^n}$ indicates how plausible $z$ is a value of $Z_{n+1}$ based on the available data $Z^n=z^n$.  Similar to Hong and Martin (2021) and Hong (2026a),  we follow the spirit of Martin and Lingham (2016) to call the function $\pl_{x_{n+1}, z^n}$ the \emph{plausibility function}, though $\pl_{x_{n+1}, z^n}$ is also called \emph{conformal $p$-value} by some authors (e.g.,  Bates et al. 2023).  Finally, we can use the plausibility function $\pl_{x_{n+1}, z^n}$ to construct a $100(1-\alpha)\%$ conformal prediction region as follows:
\begin{equation}
\label{eq:region}
C_\alpha(X_{n+1}, Z^n) = \{y: \pl_{X_{n+1}, Z^n}(y) > t_n(\alpha)\},
\end{equation}
where $0<\alpha<1$, $t_n(\alpha) = (n+1)^{-1}\lfloor (n+1)\alpha \rfloor$,  and $\lfloor a \rfloor$ denotes the greatest integer less than or equal to $a$. The basic properties of the rank statistic imply the next theorem.

\begin{thm}
\label{thm:valid}
Suppose $Z_1,Z_2,\ldots$ is a sequence of exchangeable random vectors and each $Z_i$  is generated from a distribution $\prob$.  Let  $\prob^{n+1}$ denote the corresponding joint distribution of $Z^{n+1}=\{Z_1,\ldots,Z_n,Z_{n+1}\}$.  For $\alpha \in (0,1)$, define $t_n(\alpha) = (n+1)^{-1}\lfloor (n+1)\alpha \rfloor$.  Then
\begin{equation*}
\label{eq:valid}
\sup \prob^{n+1}\{ \pl_{X_{n+1}, Z^n}(Y_{n+1}) \leq t_n(\alpha) \} \leq \alpha \quad \text{for all $n$ and all $\alpha \in (0,1)$},
\end{equation*}
where the supremum is over all distributions $\prob$ for $Z_1$.
\end{thm}

It follows from Theorem~\ref{thm:valid}  that the prediction region given by (\ref{eq:region}) is \emph{finite-sample valid} in the sense that
\begin{equation}
\label{eq:jointvalidity}
\prob^{n+1}\{Y_{n+1}\in C_{\alpha}(X_{n+1},  Z^n) \}\geq 1-\alpha\quad \text{for all $(n, \prob)$},
\end{equation}
where $\prob^{n+1}$ is the joint distribution for $(X_1, Y_1), \ldots, (X_n, Y_n), (X_{n+1}, Y_{n+1})$.   This finite-sample validity says the coverage probability of the conformal prediction region is no less than the advertised confidence level.

While (\ref{eq:jointvalidity}) guarantees that the conformal prediction region $C_\alpha(X_{n+1}, Z^n)$ is finite-sample valid for \emph{any} non-conformity measure $M$,  a data scientist faces several challenges in practice.  First,  it is clear from (\ref{eq:region}) that exact determination of a conformal prediction region generally requires a data scientist to run Algorithm~\ref{algo:conformal.s} for \emph{all} possible $y\in\mathbb{R}$.  This is practically impossible since there are infinitely many possible $y$-values. To be clear,  by exact determination we mean the process of determining $C_{\alpha}(X_{n+1}, Z^n)$ must not involve any approximation.  Clearly, a brute-force approach for running Algorithm~\ref{algo:conformal.s} can only be done for a grid of $y$ values.  Let $\widehat{C}_{\alpha}(X_{n+1},  Z^n)$ denote the prediction set results from a brute-force approach.  Then the finite-sample validity of $\widehat{C}_{\alpha}(X_{n+1},  Z^n)$ is no longer guaranteed by (\ref{eq:jointvalidity}).  It is important to underline that this challenge is not to be confused with the finite-sample validity of the conformal prediction region $C_{\alpha}(X_{n+1}, Z^n)$. The finite sample validity of $C_{\alpha}(X_{n+1}, Z^n)$ is guaranteed by Theorem~\ref{thm:valid}.  But the set $\widehat{C}_{\alpha}(X_{n+1},  Z^n)$ is not the same as the conformal prediction region $C_{\alpha}(X_{n+1}, Z^n)$.  Therefore, one cannot use the finite-sample validity of $C_{\alpha}(X_{n+1}, Z^n)$ to justify the finite-sample validity of $\widehat{C}_{\alpha}(X_{n+1}, Z^n)$.  The second challenge is closely related to the first.  Even if we choose to consider only a grid of $y$ values,  the computation needed for determining $\widehat{C}_{\alpha}(X_{n+1},  Z^n)$ could still be prohibitively expensive.  Finally,   the prediction region $C_{\alpha}(X_{n+1}, Z^n)$ is not guaranteed to be an interval.  In general,  $C_{\alpha}(X_{n+1}, Z^n)$ can be a disjoint union of several non-overlapping intervals (e.g., Lei et al.  2013),  which is inappropriate for many applications.   For practical purposes, a data scientist often needs a prediction region to be an interval of a certain shape,  such as  $(-\infty, a)$, $(-b,  b)$, or $(c, \infty)$, where $a, b$, and $c$ are real numbers and $b>0$. We will only aim at these three shapes for $C_{\alpha}(X_{n+1}, Z^n)$ in this article.  Among these challenges, the first is the most serious one,  because it can compromise the finite-sample validity of predictions---the key selling point of conformal prediction.

\subsection{Extant approaches to addressing practical challenges of conformal prediction}

In the existing literature, there are four major approaches to addressing the above challenges.  Here we review these approaches in chronological order. 

The first approach is to seek an easy-to-compute equivalent form of the conformal prediction set $C_{\alpha}(X_{n+1}, Z^n)$.  This approach was pioneered by Vovk et al. (2005) and reviewed in Burnaev and Vovk (2014).  These authors used the absolute ridge regression residual as a non-conformity measure and showed that the resulting conformal prediction set $C_{\alpha}(X_{n+1}, Z^n)$ has an equivalent form that be determined exactly (e.g., Vovk et al, Section~2.3; Burnaev and Vovk 2014).  One drawback of their proposal is that their conformal prediction set $C_{\alpha}(X_{n+1}, Z^n)$ is not necessarily an interval, and the sufficient condition for their conformal prediction set to be an interval is data-dependent (e.g., Vovk et al.,  Pages~31--32) and beyond the control of a data scientist.  Lei (2019) applied linear homotopy, a tool in algebraic topology, to find equivalent set of $C_{\alpha}(X_{n+1}, Z^n)$ when the non-conformity measure is absolute lasso residual,  and showed that $C_{\alpha}(X_{n+1}, Z^n)$ in this case is not necessarily an interval.  He also established that $C_{\alpha}(X_{n+1}, Z^n)$ would be an interval if we replace the absolute lasso residual with the absolute elastic net residual.  Li (2024) extended these results to a general non-conformity measure.  However,  no closed-form formula for lasso or elastic net solution is known.   This means a practitioner must resort to numerical approximation to implement the ``conformal lasso algorithm'' in Lei (2019), and the finite-sample validity of the resulting prediction set is no longer guaranteed.  The same can be said for the ``fast conformalization algorithm'' in Li (2024), since we generally need numerical approximation in solving ODEs (ordinary differential equations).  Therefore, while  Lei (2019) and Li (2024) made important contributions by chartering a new course towards the exact determination of $C_{\alpha}(X_{n+1}, Z^n)$,  further research is required to either (i) obtain a closed-form formula for the solutions to lasso, elastic net, and ODEs, or (ii) show why approximation used in numerical solution to lasso, elastic net, and ODEs does not affect the finite-sample validity of the resulting prediction set.  In unsupervised learning, Hong and Martin (2021) showed that if we take the non-conformity measure to be the empirical CDF (cumulative distribution function)  and the kernel of the empirical CDF is constant, then the resulting conformal prediction set coincides with the one-sided prediction interval based on order statistics (e.g., Wilks 1941; Fligner Wolfe 1976) and can be determined exactly.   Under the constraint that $Y$ is non-negative,  Hong (2026a) showed that for a specific non-conformity measure,  $C_{\alpha}(X_{n+1}, Z^n)$ can be determined exactly and is a one-sided interval.

The second approach is initiated by Lei et al. (2013, 2014).  In this approach,  we obtain a sandwich approximation set $C^+_{\alpha}(X_{n+1}, Z^n)$ of $C_{\alpha}(X_{n+1}, Z^n)$ such that (i) $C_{\alpha}(X_{n+1}, Z^n)\subseteq C^+_{\alpha}(X_{n+1}, Z^n)$,  and (ii) $C^+_{\alpha}(X_{n+1}, Z^n)$ can be determined exactly.  Since $C_{\alpha}(X_{n+1}, Z^n)$ is finite-sample valid, so is $C^+_{\alpha}(X_{n+1}, Z^n)$.  One potential drawback of this approach is that we do not know how conservative $C^+_{\alpha}(X_{n+1}, Z^n)$ is relative to $C_{\alpha}(X_{n+1}, Z^n)$ especially when the sample size is small or modest.  Lei et al (2013, 2014) obtained some convergence rate results. However, since every convergence rate result is up to an unknown constant,  these results do not tell us the precise degree of conservativeness of $C^+_{\alpha}(X_{n+1}, Z^n)$. In addition,  kernel density estimators are used to create the non-conformity measures in Lei et al.  (2013, 2014).  It is unclear to find a superset of $C^+_{\alpha}(X_{n+1}, Z^n)$ for an arbitrary non-conformity measure. 

The third approach, called \emph{split conformal prediction},  was proposed by Lei et al. (2018), though similar ideas appeared under the name of \emph{inductive conformal inference} in Papadopoulos et al.  (2002) and Vovk et tal. (2005).  In this approach,  we first randomly split the data into two equal-sized parts: a training set and a test set; then, we fit a regression model, such as ridge regression or lasso using the training data set; next, we apply the prediction interval based on order statistics (e.g., Wilks 1941; Fligner Wolfe 1976) to the absolute test predictive errors to obtain a finite-sample valid $100(1-\alpha)\%$ prediction interval $J_{\alpha}(X_{n+1}, Z^n)$  for the next absolute test predictive error; finally,  we converts $J_{\alpha}(X_{n+1}, Z^n)$  to a finite-sample valid $100(1-\alpha)\%$ prediction set for the response variable $Y$.  Split conformal prediction is widely applicable since the data scientist is free to choose a regression model,  and a split conformal prediction set is guaranteed to be a finite-sample valid interval.  For example, if we replace the predicted value of $Y$ by the value of the quantile function at $X_i$ in the split conformal prediction, we will obtain the quantile regression conformal prediction set proposed in Romano et al. (2019).  However,  if the chosen regression model is wrong, each test absolute predictive error will likely be large, rendering the resulting split conformal prediction interval conservative.  To avoid confusion, we will refer to the original conformal prediction framework as \emph{full conformal prediction}.  When no confusion might arise, we might suppress the word ``full''.  Therefore,  ``conformal prediction'' means ``full conformal prediction'',  unless otherwise stated.

The fourth approach is based on a general strategy advanced by Hong (2026b). In this approach, we first choose a non-constant $p$-variate real-valued function $h$, called a \emph{transformation},  and write (\ref{eq:true}) as 
\[
Y=h(X)+W,
\]
where $W=Y-h(X)$.  Then,  we obtain $W_i=Y_i-h(X_i), \ i=1, \ldots, n$ using data $(X_1, Y_1), \ldots, (X_n, Y_n)$. Next, we obtain a finite-sample valid $100(1-\alpha)\%$ prediction interval $I_{\alpha}(X_{n+1}, Z^n)$ for $W_{n+1}$, using order statistics (e.g., Wilks 1941; Fligner Wolfe 1976). Finally, we transfer $I_{\alpha}(X_{n+1}, Z^n)$  to a finite-sample valid $100(1-\alpha)\%$ for $Y_{n+1}$ by a shift of $h(X_{n+1})$.  We will call this approach \emph{transformation method}.  Similar to split conformal prediction,  the transformation method is also widely applicable,  since a data scientist can choose any transformation $h$,  and the resulting prediction set is always a finite-sample valid interval.  However,  no theoretical result is available for choosing the transformation $h$, though empirical evidence in Hong (2026b) suggests that if $h$  grows no faster than the linear function in each of its variables, the resulting prediction interval will be relatively tight.

Note that only the first two approaches are within the full conformal prediction framework.  Both split conformal prediction and the transformation method have moved away from full conformal prediction. In particular, no leave-one-out action is performed to compute a non-conformity score in these two methods.  

\subsection{Our contributions}

This article follows the first approach to investigate how we might overcome the above challenges in the full conformal prediction framework.  Following Vovk et al. (2019),  we say a non-conformity measure $M$ is  \emph{monotonically increasing} if 
\[
y\leq y' \Longrightarrow M(B, (x, y) )\leq M(B, (x, y'))\quad \text{for all $x\in \mathbb{R}^p$}; 
\]
we say $M$ is \emph{monotonically decreasing} if 
\[
y\leq y' \Longrightarrow M(B, (x, y))\geq M(B, (x, y'))\quad \text{for all $x\in \mathbb{R}^p$}. 
\]
We say  $M$ is \emph{monotonic} if it is either monotonically increasing or monotonically decreasing.  Monotonicity of a non-conformity in unsupervised learning is defined similarly; see the Appendix.  Hong (2026a) found a suitable monotonic non-conformity measure $M$ so that 
\begin{equation}
\label{eq:property*}
\mu_i\geq \mu_{n+1}  \text{ if and only if $Y_i\geq f(X^{n+1})+Y_{n+1}$ or $Y_i\leq f(X^{n+1})+Y_{n+1}$,}
\end{equation} 
where $X^{n+1}=\{X_1, \ldots, X_n, X_{n+1}\}$ and $f$ is some real-valued function. Then,  (\ref{eq:property*}) implies that the corresponding plausibility function is monotonic in $y$,  which further implies the resulting conformal prediction region $C_{\alpha}(X_{n+1}, Z^n)$ equals the prediction interval based on some order statistics.  The ad hoc strategy employed by Hong and Martin (2021) for unsupervised learning is of the same spirit; their non-conformity measure  is also monotonic, and their non-conformity measure leads to 
\begin{equation}
\label{eq:monotonic}
\mu_i\geq \mu_{n+1} \text{ if and only if $X_i\geq X_{n+1}$  or $X_i\leq X_{n+1}$},
\end{equation}
where $X_i\in\mathbb{R}$ for $1\leq i \leq n+1$.
In both cases,  the three aforementioned challenges are overcome simultaneously: the determination of $C_{\alpha}(X_{n+1}, Z^n)$ is exact, the computation needed is simple, and the shape of the $C_{\alpha}(X_{n+1}, Z^n)$ is an interval.  In general, if we can determine $C_{\alpha}(X_{n+1}, Z^n)$ exactly,  then the computational challenge will likely vanish, though care is still needed to ensure $C_{\alpha}(X_{n+1}, Z^n)$ is an interval of a desired form.

Given the above observations regarding Hong and Martin (2021) and Hong (2026a), it is natural to ask whether we can say something about the relationship among the monotonicity of $M$, the monotonicity of the plausibility function,  Property~(\ref{eq:property*}),  the shape of the conformal prediction region, and exact determination of the conformal prediction region $C_{\alpha}(X_{n+1}, Z^n)$.  In particular,  there are two sets of open questions.

%%%
\subsubsection*{Questions regarding monotonicity of $M$}
 
\begin{enumerate}
\item[(I)]Does monotonicity of the non-conformity measure $M$ imply (\ref{eq:property*})?
\item[(II)]Does (\ref{eq:property*}) imply the monotonicity of $M$? 
\item[(III)]Is monotonicity of $M$ a necessary condition for $C_{\alpha}(X_{n+1}, Z^n)$ to be an interval?
\item[(IV)]Does the monotonicity of $M$ imply that the resulting conformal prediction region is an interval?
\item[(V)]Is the monotonicity of $M$ a necessary condition for the exact determination of $C_{\alpha}(X_{n+1}, Z^n)$?
\end{enumerate}

%%%
\subsubsection*{Questions regarding general properties of full conformal prediction}
 
\begin{enumerate}
\item[(VI)]Is (\ref{eq:property*}) a necessary condition for $C_{\alpha}(X_{n+1}, Z^n)$ to be an interval? (Note that the converse is true: (\ref{eq:property*}) implies $C_{\alpha}(X_{n+1}, Z^n)$ is a one-sided interval.)
\item[(VII)]Is monotonicity of the plausibility function $\pl_{x_{n+1}, z^n}$ a necessary condition for $C_{\alpha}(X_{n+1}, Z^n)$ to be a one-sided interval? (Note that the converse holds, i.e.,  monotonicity of the plausibility function implies $C_{\alpha}(X_{n+1}, Z^n)$ is a one-sided interval.)
\item[(VIII)]Is (\ref{eq:property*}) a necessary condition for the exact determination of $C_{\alpha}(X_{n+1}, Z^n)$? (Note that the converse is true: (\ref{eq:property*}) implies $C_{\alpha}(X_{n+1}, Z^n)$ is a one-sided interval; hence it implies exact determination of $C_{\alpha}(X_{n+1}, Z^n)$.)
\item[(IX)]Is monotonicity of the plausibility function $\pl_{x_{n+1}, z^n}$ a necessary condition for the exact determination of $C_{\alpha}(X_{n+1}, Z^n)$? (Note that the converse holds, i.e.,  monotonicity of the plausibility function implies $C_{\alpha}(X_{n+1}, Z^n)$ is a one-sided interval; hence, it implies the exact determination of $C_{\alpha}(X_{n+1}, Z^n)$.)
\end{enumerate}

The contributions of this article are multi-fold.  First,  we offer new insights into full conformal prediction by answering the above questions.  These new insights suggest that it is challenging to find a hard-and-fast rule for choosing a non-conformity measure.  In view of this fact and the principle of parsimony,  we then propose a flexible quadratic polynomial non-conformity measure.  We show that by setting the parameters of this quadratic polynomial non-conformity measure to certain values,  the corresponding $100(1-\alpha)\%$ conformal prediction interval $C_{\alpha}(X_{n+1}, Z^n)$ is  a one-sided interval or two-sided interval, and these intervals can be determined easily and exactly.  In particular, the (inflexible) linear non-conformity measure proposed in Hong (2026a) is a special case of our (flexible) quadratic polynomial non-conformity measure.  Therefore,  our proposed non-conformity measure is one of the rare ones, if not the first one,  in the full conformal prediction framework that leads to exact determination of $C_{\alpha}(X_{n+1}, Z^n)$,  and allows a data scientist to choose whether  $C_{\alpha}(X_{n+1}, Z^n)$ is $(-\infty, a)$-shaped, $(a, \infty)$-shaped, or $(a,  b)$-shaped.  In addition, we establish an interesting result regarding the efficiency of a bounded finite-sample prediction interval: there is no bounded finite-sample valid non-parametric $100(1-\alpha)\%$ prediction interval (based on conformal prediction or not) with the shortest mean length for any confidence level $1-\alpha$.

The remainder of the paper is organized as follows.  Section~2 answers Questions~(I)--(IX).  Section~3 details our proposal.  Section~4 provides several numerical examples,  based on simulation, to demonstrate the performance of the proposed method.   Section~5 concludes the article with some remarks.  The insights in Section~2 and the strategy in Section~3 have their counterparts in unsupervised learning, which are given in the Appendix.

\section{Answers to Questions~(I)---(IX)}

Here,  we answer Questions~(I)---(IX).  If we answer an open question in the affirmative,  we will give a proof; otherwise, we will give a counterexample.  Henceforth,  we will use the following notation.  Let $B=\{z_1, \ldots, z_n\}$ be a bag of  observations of size $n$. For $i=1, \ldots, n$, let $z_i=(x_{i1}, \ldots, x_{ip}, y_i)$ be the $i$-th observation in $B$.  That is, $x_{ij}$ denotes the $i$-th observation of the $j$-th feature.   $z=(x_1, \ldots, x_p, y)$ will denote a provisional value of a future observation to be predicted.

\subsection{Answers to questions regarding monotonicity of $M$}

\noindent \textbf{Answer to Question (I): No.  Monotonicity of $M$ need not imply (\ref{eq:property*}).}

\begin{example}
\label{example1.1}
Let $p=1$ and $M(B, z)=(\sum_{j=1}^n x_{j1}+x)+\min\{y_1, \ldots, y_n\}+y$.  Then $M$ is monotonic.  Also,
\begin{eqnarray*}
\mu_i &=& M(Z^{n+1}\backslash Z_i, Z_i) \\
	&=& \sum_{j=1}^{n+1}X_{j1}+ \min\{Y_1, \ldots, Y_{i-1}, Y_{i+1}, \ldots, Y_n, Y_{n+1}\}+Y_i, \quad i=1, \ldots, n, n+1.
\end{eqnarray*}
Thus, 
\begin{eqnarray}
\label{eq:example1.1}
\mu_i \geq \mu_{n+1} &\Longleftrightarrow & \min\{Y_1, \ldots, Y_{i-1}, Y_{i+1}, \ldots, Y_n, Y_{n+1}\}+Y_i\nonumber \\
                                     & &\geq  \min\{Y_1, \ldots, Y_n\}+Y_{n+1} \nonumber\\
                                     &\Longleftrightarrow & \min\{m_i,Y_{n+1}\}+Y_i\geq \min\{m_i, Y_i\}+Y_{n+1},
\end{eqnarray}
where $m_i= \min\{Y_1, \ldots, Y_{i-1}, Y_{i+1}, \ldots, Y_n\}$.
If $m_i \geq Y_{n+1}$,  then the last inequality of (\ref{eq:example1.1}) becomes
\[
Y_{n+1}+Y_i\geq \min\{m_i, Y_i\}+Y_{n+1},
\]
which is trivially true.  When $m_i<Y_{n+1}$,  the last inequality of (\ref{eq:example1.1}) is equivalent to
\[
m_i+Y_i\geq \min\{m_i, Y_i\}+Y_{n+1},
\]
i. e., $m_i<Y_{n+1}\leq Y_i+m_i-\min\{m_i, Y_i\}$.  Now if $m_i\leq Y_i$,  then $m_i<Y_{n+1}\leq Y_i$.  However,  if $m_i> Y_i$ we would have $m_i<Y_{n+1}\leq m_i$, which is absurd.  Therefore,  (\ref{eq:property*}) does not hold in this case.
\end{example}

\bigskip

\noindent \textbf{Answer to Question (II): No.  (\ref{eq:property*}) does not imply $M$ is monotonic.}

\begin{example}
\label{example1.2}
Let $p=1$ and $M(B, z)=(\sum_{j=1}^n x_{j1}+x)+(y_1^2+\ldots+y_n^2)+y^2+y$.  Clearly, $M$ is not monotonic. We have
\[
\mu_i=M(Z^{n+1}\backslash\{Z_i\}, Z_i)=\sum_{j=1}^{n+1}X_{j1}+\sum_{j=1, j\neq i}^{n+1}Y_j^2+Y_i^2+Y_i, \quad j=1, \ldots, n, n+1.
\]
Hence,  $\mu_i\geq \mu_{n+1}$ if and only if 
\[
\sum_{j=1, j\neq i}^{n+1}Y_j^2+Y_i^2+Y_i \geq \sum_{j=1}^n Y_j^2+Y_{n+1}^2+Y_{n+1},
\]
which is equivalent to $Y_{n+1}\leq Y_i$. Therefore, (\ref{eq:property*}) holds.
\end{example}

\bigskip

\noindent \textbf{Answer to Question (III): No.  Monotonicity of $M$ is not a necessary condition for $C_{\alpha}(X_{n+1}, Z^n)$ to be an interval?}

\begin{example}
Consider Example~\ref{example1.2}.  $M$ is not monotonic.  However, (\ref{eq:property*}) holds, which implies $C_{\alpha}(X_{n+1}, Z^n)=(-\infty, Y_{(k)})$, where $Y_{(k)}$ is the $k$-th order statistic of $Y_1, \ldots, Y_n$.
\end{example}

\bigskip

\noindent \textbf{Answer to Question (IV): No.  Monotonicity of $M$ need not imply $C_{\alpha}(X_{n+1}, Z^n)$ is an interval.}

\begin{example}
\label{exampleIV}
Let  $p=1$ and $M(B, z)=(\sum_{j=1}^n x_{j1}+x)+y_1^2+\ldots+y_n^2+y$. Then $M$ is  monotonic.  We have
\[
\mu_i= M(Z^{n+1}\backslash Z_i, Z_i) =\sum_{j=1}^{n+1}X_{j1}+\sum_{j=1, j\neq i}^{n+1}Y_j^2+Y_i, \quad i=1, \ldots, n, n+1.
\]
Thus,  
\[
\mu_i\geq \mu_{n+1} \Longleftrightarrow \sum_{j=1, j\neq i}^{n+1}Y_j^2+Y_i\geq \sum_{j=1}^{n}Y_j^2+Y_{n+1}.
\]
It follows that
\[
\mu_i\geq \mu_{n+1} \Longleftrightarrow Y_{n+1}^2+Y_i\geq Y_i^2+Y_{n+1},
\]
implying
$Y_{n+1}\in (-\infty,  \min\{Y_i, 1-Y_i\})\cup (\max\{Y_i, 1-Y_i\}, \infty)$.
\end{example}

\bigskip

\noindent \textbf{Answer to Question (V): No.  Monotonicity of $M$ is not a necessary condition for the exact determination of $C_{\alpha}(X_{n+1}, Z^n)$.} 

\begin{example}
Consider Example~\ref{example1.2}.  In this case,  $C_{\alpha}(X_{n+1}, Z^n)=(-\infty, Y_{(k)})$. Thus, we can determine $C_{\alpha}(X_{n+1}, Z^n)$ exactly, though $M$ is not monotonic.
\end{example}
\noindent \textbf{Remark.} We did not ask the converse of Question~(V), i.e., whether the monotonicity of $M$  implies that $C_{\alpha}(X_{n+1}, Z^n)$ can be determined exactly, because that question seems to be too broad to be well-defined. The next example illustrates this point.

\begin{example}
\label{ex:converseIX}
Let $p=1$ and $M(B, z)=\sum_{i=1}^{n+1}x_{i1}+x+e^{(\max\{0, y\})^8}+\max\{0, y\}$. Then $M$ is monotonically increasing.  In this case, 
\[
\mu_i= M(Z^{n+1}\backslash Z_i, Z_i) =\sum_{j=1}^{n+1}X_{j1}+e^{(\max\{0, Y_i\})^8}+\max\{0, Y_i\}, \quad i=1, \ldots, n, n+1.
\]
In particular, we have
\[
\mu_{n+1}= M(Z^{n+1}\backslash Z_i, Z_i) =\sum_{j=1}^{n+1}X_{j1}+e^{(\max\{0, Y_{n+1}\})^8}+\max\{0, Y_{n+1}\}.
\]
It follows that $\mu_i\geq \mu_{n+1}$ if and only if
\[
e^{(\max\{0, Y_{n+1}\})^8}+\max\{0, Y_{n+1}\}-\left[e^{(\max\{0, Y_i\})^8}+\max\{0, Y_i\}\right] \leq 0,
\]
which is not known to have any closed-formula solutions. Therefore,  we do not know any method for determining $C_{\alpha}(X_{n+1}, Z^n)$ exactly. This does not mean we cannot find such a method in the future, nor does it imply that such a method does not exist. 
\end{example}

\subsection{Questions regarding general properties of full conformal prediction}

\noindent \textbf{Answer to Question (VI): No.  (\ref{eq:property*}) is not a necessary condition for $C_{\alpha}(X_{n+1}, Z^n)$ to be an interval.}

\begin{example}
\label{exampleV}
Consider Example~\ref{example1.1}.  We already know that (\ref{eq:property*}) does not hold in this case.  Now note that (\ref{eq:example1.1}) implies that 
\[
C_{\alpha}(X_{n+1}, Z^n)=(-\infty, a_{(k)}),
\]
where $a_i= Y_i+\min\{Y_1, \ldots, Y_{i-1}, Y_{i+1}, \ldots, Y_n\}-\min\{Y_1, \ldots, Y_n\}$. Therefore,  $C_{\alpha}(X_{n+1}, Z^n)$ is a one-sided interval. 
\end{example}

\bigskip

\noindent \textbf{Answer to Question (VII): Yes.  Monotonicity of the plausibility function $\pl_{x_{n+1}, z^n}$ is a necessary condition for $C_{\alpha}(X_{n+1}, Z^n)$ to be a one-sided interval.}

\begin{proof}
We will prove this fact by contradiction.  Suppose $C_{\alpha}(X_{n+1}, Z^n)$ is a one-sided interval.  Without loss of generality,  we assume the plausibility function $\pl_{x_{n+1}, z^n}$ is not monotonically increasing.  Then, there exist three numbers $a<b<c$ such that $\pl_{x_{n+1}, z^n}(b)>\pl_{x_{n+1}, z^n}(a)$ and $\pl_{x_{n+1}, z^n}(b)>\pl_{x_{n+1}, z^n}(c)$. Thus, for any confidence level $\alpha$ such that $\max\{\pl_{x_{n+1}, z^n}(a), \pl_{x_{n+1}, z^n}(c)\}<\alpha<\pl_{x_{n+1}, z^n}(b)$, we will have $b\in C_{\alpha}(X_{n+1}, Z^n)$ but $a\not\in C_{\alpha}(X^n)$ and $c\not\in C_{\alpha}(X_{n+1}, Z^n)$. Therefore, $C_{\alpha}(X_{n+1}, Z^n)$ cannot be a one-sided interval.
\end{proof}

\bigskip

\noindent \textbf{Answer to Question (VIII): No.  (\ref{eq:property*}) is not a necessary condition for the exact determination of $C_{\alpha}(X_{n+1}, Z^n)$.} 

\begin{example}
Consider Example~\ref{exampleV}.
\end{example}

\bigskip

\noindent \textbf{Answer to Question (IX): No.  Monotonicity of the plausibility function $\pl_{x_{n+1}, z^n}$ is not a necessary condition for the exact determination of $C_{\alpha}(X_{n+1}, Z^n)$.} 

\begin{example}
Consider Example~\ref{exampleIV}.
\end{example}

\section{Proposed strategy}

First, we make a simple but important observation.

\begin{thm}
\label{thm:predproperty}
If $\lfloor (n+1)\alpha \rfloor< 1$, then the $(1-\alpha)\%$ prediction region $C_{\alpha}(X_{n+1}, Z^n)$ given by (\ref{eq:region}) is $\mathbb{R}$.
\end{thm}

\begin{proof}
Since $\lfloor (n+1)\alpha \rfloor<1$,  we have $Y_{n+1}\in C_{\alpha}(X_{n+1}, Z^n)$ if and only if $\pl_{X_{n+1}, Z^n}(Y_{n+1})>\lfloor (n+1)\alpha \rfloor/(n+1)$ if and only if $\sum_{i=1}^{n+1}1_{\{\mu_i\geq \mu_{n+1}\}}>\lfloor (n+1)\alpha \rfloor$.  In view of the fact that $1_{\{ \mu_{n+1}\geq \mu_{n+1}\} }=1$,  we conclude that $C_{\alpha}(X_{n+1}, Z^n)=\mathbb{R}$.
\end{proof}
Thus, if we want a nontrivial conformal prediction region (i.e., $C_{\alpha}(X_{n+1}, Z^n)\neq \mathbb{R}$),   we must require $\lfloor (n+1)\alpha \rfloor\geq 1$.  For the remainder of this article,  we let $r_1=\min\{n, \lfloor (n+1)(1-\alpha)\rfloor+1\}$, $r_2=(n+1)-\lfloor(n+1)(1-\alpha)\rfloor$, and $r_3=\max\{1, \lfloor(n+1)\alpha\rfloor\}$.

The negative answers to most questions in the previous section show that it is challenging to give a hard-and-fast rule for choosing a non-conformity measure so that we can overcome the aforementioned challenges.  In the extant literature,  many existing statistical models, such as ordinary linear regression, ridge regression, lasso,  and kernel density estimation, have been used to create non-conformity measures ( e.g., Vovk et al.  2005;  Lei et al.  2013, 2014; Lei 2019).  As discussed in Section~1,  in the full conformal prediction framework,  hardly any non-conformity measure created using a complicated model allows us to both determine $C_{\alpha}(X_{n+1}, Z^n)$ exactly and control its shape.  We uphold the philosophy that estimation and prediction warrant different approaches.  For example,  a wrong model can perform predictions well and even outperform the true data-generating model in some cases, though the same cannot be said for estimation (Shmueli 2010).  We treat conformal prediction as a purely model-free method.  Therefore,  we do not use any statistical model to create a non-conformity measure.  Our philosophy does not apply to split conformal prediction,  where a complicated model is often employed to make the test absolute regression residuals small so that the split conformal prediction interval is tight.  It is evident that if (\ref{eq:property*}) holds, then we can address the three practical challenges simultaneously.  There are numerous choices of the non-conformity measures that can lead to (\ref{eq:property*}).  Driven by the principle of parsimony, we prefer something simple, such as a linear polynomial in data,  as in Hong (2026a).  However,  to obtain a two-sided/bounded prediction interval, we will need something other than (\ref{eq:property*}).  In fact, a multivariate quadratic polynomial of data will suffice, as we will see below.  Intuitively,  a higher order polynomial is likely to cause the solution to the inequality $\mu_i\geq \mu_{n+1}$ to be a union of disjoint intervals,  rendering $C_{\alpha}(X_{n+1}, Z^n)$ to be a union of disjoint intervals too.  This intuition is clear from the calculation in Example~\ref{ex:converseIX} and the proofs of Theorem~\ref{thm:interval_rightbounded} and Theorem~\ref{thm:interval_two}.  Therefore, we propose the following non-conformity measure.

For $i=1, \ldots, n$,  we write $X_i$ as $(X_{i1}, \ldots, X_{ip})$ where each $X_{ij}\in \mathbb{R}$.  That is, $X_{ij}$ denotes the $i$-th observation of the $j$-th predictor.  Hence,  $(X_{i1}, \ldots, X_{ip}, Y_i)=(X_i, Y_i)$ is the $i$-th observation.   Suppose $B=\{(x_{11}, \ldots, x_{1p}, y_1), \ldots, (x_{n1}, \ldots, x_{np}, y_n)\}$ and  $z=(x_1, \ldots, x_p, y)$ is a provisional value of $Z_{n+1}$.  Consider the following nonconformity measure
\begin{equation}
\label{eq:ncmeasure}
M(B, z)=(\beta_2 y^2+\beta_1 y)+\left[\gamma \sum_{j=1}^px_j+ \sum_{i=1}^n \left(\eta y_i-\delta\sum_{j=1}^px_{ij} \right)\right],
\end{equation}
where $\beta_1, \beta_2,  \gamma$,  $\delta$, and $\eta$ are constant parameters to be chosen at the discretion of the data scientist.  Indeed,  the main function of these parameters is to control the shape of the conformal prediction regions.  For example, if $\beta_2=0$ and $\beta_1=1$, then the resulting conformal prediction region $C_{\alpha}(X_{n+1}, Z^n)$ will be a $(-\infty, a)$-shaped interval; if $\beta_2=0$ and $\beta_1=-1$,  then $C_{\alpha}(X_{n+1}, Z^n)$ will be a $(a, \infty)$-shaped interval; when $\beta_2\neq 0$,  $\beta_1=0$,  $\eta=0$, and $\gamma=-1$, the resulting conformal prediction region will be a $(a, b)$-shaped interval.  The parameters $\gamma$, $\eta$, and $\delta$ are the weights a data scientist attaches to the information in $Y_1, \ldots, Y_n$, $X_1, \ldots, X_n$, and $X_{n+1}$, respectively.  The next three theorems show that this non-conformity measure enables us to address the aforementioned three challenges.

\begin{thm}
\label{thm:interval_rightbounded}
Suppose $0<\alpha<1$, $\lfloor (n+1)\alpha \rfloor\geq 1$,  and the non-conformity measure is given by (\ref{eq:ncmeasure}).   If $\beta_2=0$, $\beta_1=1$,  $\gamma, \delta\in \mathbb{R}$, and $1-\eta>0$, then the $100(1-\alpha)\%$ conformal prediction region $C_{\alpha}(X_{n+1}, Z^n)$ is the one-sided interval $(-\infty, a_{(r_1)})$, where  $a_i=Y_i+ \frac{(\delta+\gamma)}{1-\eta}\sum_{j=1}^p (X_{ij}-X_{(n+1) j})$ for $1\leq i\leq n$ and $a_{(k)}$ is the $k$-th ordered value of $a_1, \ldots, a_n$. 
\end{thm}

\begin{proof}
For $i=1, \ldots, n, n+1$, let $S_i$ denote the sum $\sum_{j=1}^p X_{ij}$. Then 
\begin{eqnarray*}
\mu_i &=& M(Z^{n+1}\backslash \{Z_i\}, Z_i)=Y_i + \left[\gamma S_i+\sum_{j=1, j\neq i}^{n+1}( \eta Y_j-\delta S_j)\right],\ i=1, \ldots, n, \\
\mu_{n+1} &=& M(Z^n, Z_i)=Y_{n+1}+\left[\gamma  S_{n+1}+\sum_{j=1}^n(\eta Y_j-\delta  S_j)\right].
\end{eqnarray*}
Therefore, $\mu_i\geq \mu_{n+1}$ if and only if 
\[
Y_{n+1}+\gamma  S_{n+1}+\sum_{j=1}^n(\eta Y_j-\delta  S_j)\leq Y_i+\gamma S_i+\sum_{j=1, j\neq i}^{n+1}(\eta Y_j-\delta S_j),
\]
which is equivalent to 
%\[
%Y_{n+1}+\gamma  S_{n+1}+\sum_{j=1}^n(\eta Y_j-\delta  S_j)\leq Y_i+\gamma S_i+\sum_{j=1}^n(\eta Y_j-\delta S_j)+(\eta Y_{n+1}-\delta S_{n+1})-(\eta Y_i-\delta S_i),
%\]
\[
(1-\eta)Y_{n+1}\leq (1-\eta)Y_i+(\gamma+\delta)(S_i-S_{n+1}).
\]
Since $1-\eta>0$,  the last display implies $\mu_i\geq \mu_{n+1}$ if and only if $Y_{n+1}\leq Y_i+\frac{(
\delta+\gamma)}{1-\eta}(S_i-S_{n+1})$ for $1\leq i\leq n$.  Therefore,  the theorem follows from   (\ref{eq:region}) and the definition of the plausibility function $\pl_{x_{n+1}, z^n}$.
\end{proof}
\noindent \textbf{Remark.} If we disregard the information of $X$ and rely on information of $Y$ only,  we know $(-\infty, Y_{(r_1)})$ is a finite-sample valid $100(1-\alpha)\%$ prediction interval for $Y$ (e.g., Wilks 1941; Fligner and Wolfe 1976).  We believe the additional information provided by $X$ is generally useful, and is represented by the term $\frac{(\delta+\gamma)}{1-\eta}\sum_{j=1}^p (X_{ij}-X_{(n+1) j})$ in this theorem.  Note that choosing the weight $\frac{(\delta+\gamma)}{1-\eta}<1$ will not necessarily shorten $C_{\alpha}(X_{n+1}, Z^n)$ in this case, since $\sum_{j=1}^p (X_{ij}-X_{(n+1) j})$ can be either positive or negative.  A similar remark applies to the next two theorems.

\begin{thm}
\label{thm:interval_leftbounded}
Suppose $0<\alpha<1$, $\lfloor (n+1)\alpha \rfloor\geq 1$,  and the non-conformity measure is given by (\ref{eq:ncmeasure}).   If $\beta_2=0$, $\beta_1=-1$,  $\gamma, \delta\in \mathbb{R}$,  and $1-\eta < 0$,  then the $100(1-\alpha)\%$ conformal prediction region $C_{\alpha}(X_{n+1},  Z^n)$ is the one-sided interval $( a_{(r_2)}, \infty)$, where  $a_i=Y_i+ \frac{(\delta+\gamma)}{\eta-1}\sum_{j=1}^p (X_{ij}-X_{(n+1) j})$ for $1\leq i\leq n$.
\end{thm}

\begin{proof}
The proof is completely similar to that of Theorem~\ref{thm:interval_rightbounded}.
\end{proof}

\begin{thm}
\label{thm:interval_two}
Suppose $0<\alpha<1$, $\lfloor (n+1)\alpha \rfloor\geq 1$,  and the non-conformity measure is given by (\ref{eq:ncmeasure}).   If $\beta_2=1$,  $\beta_1=0$, $\gamma=-1$,  $\eta=0$,  and $\delta$ satisfies the inequality
\begin{equation}
\label{eq:delta}
1+\max_{i'} \left \{Y_{i'}^2/\sum_{j=1}^p (X_{(n+1) j}-X_{i' j}) \right\} \leq \delta \leq 1+\min_{i''} \left \{Y_{i''}^2/\sum_{j=1}^p (X_{(n+1) j}-X_{i''j}) \right\},
\end{equation}
where $i', i''\in {1, \ldots, n}$ such that $\sum_{j=1}^p (X_{(n+1) j}-X_{i' j})>0$ and $\sum_{j=1}^p (X_{(n+1) j}-X_{i''j})<0$.
Then the $100(1-\alpha)\%$ conformal prediction region $C_{\alpha}(X_{n+1}, Z^n)$ is the bounded equal-tailed interval $(-a_{(r_3)}, a_{(n-r_3+1)})$, where $a_i = \sqrt{Y_i^2+(1-\delta)\sum_{j=1}^p (X_{(n+1) j}-X_{ij}) }$ for $1\leq i\leq n$.
\end{thm}

\begin{proof}
We still let $S_i$ denote the sum $\sum_{j=1}^p X_{ij}$ for $i=1, \ldots, n, n+1$. Then 
\begin{eqnarray*}
\mu_i &=& Y_i^2- S_i-\delta \sum_{j=1, j\neq i}^{n+1} S_j,\ i=1, \ldots, n, \\
\mu_{n+1} &=& Y_{n+1}^2- S_{n+1}-\delta \sum_{j=1}^n S_j.
\end{eqnarray*}
Thus, $\mu_i\geq \mu_{n+1}$ if and only if 
\begin{equation}
\label{eq:aux1}
Y_{n+1}^2\leq Y_i^2+(1-\delta)(S_{n+1}-S_i).
\end{equation}
When $\delta$ satisfies (\ref{eq:delta}),  the right-hand side of (\ref{eq:aux1}) is positive. Therefore,  $\mu_i\geq \mu_{n+1}$ if and only if $Y_{n+1}$ belongs to
\[
I_i=\left(-\sqrt{Y_i^2+(1-\delta)(S_{n+1}-S_i)}, \sqrt{Y_i^2+(1-\delta)(S_{n+1}-S_i)} \right)=(-a_i, a_i),  \ i=1, \ldots, n.
\]
Now put $I_{(i)}=(-a_{(n-i+1)}, a_{(i)})$ for $i=1, \ldots, n$. Then $I_{(1)}\subseteq I_{(2)}\subseteq \ldots \subseteq I_{(n)}$.  The same argument in the proof of Theorem~\ref{thm:predproperty} shows that $Y_{n+1}\in C_{\alpha}(X_{n+1}, Z^n)$ if and only if $\sum_{i=1}^{n+1} 1_{\{\mu_i\geq \mu_{n+1}\}}>\lfloor (n+1)\alpha\rfloor$, of equivalently,  $\sum_{i=1}^{n} 1_{\{\mu_i\geq \mu_{n+1}\}}\geq r_3$.  Note that $Y_{n+1}\in (a_{(r_3)}, b_{(n-r_3+1)})$ if and only if $Y_{n+1}$ belongs to at least $r_3$ $I_{(i)}$'s. Therefore, $Y_{n+1}\in C_{\alpha}(X_{n+1}, Z^n)$ if and only if $Y_{n+1}\in (a_{(r_3)}, a_{(n-r_3+1)})$. 
\end{proof}
\noindent \textbf{Remark.} The left-hand side of (\ref{eq:delta}) is increasing in $n$ and the right-hand side of (\ref{eq:delta}) is decreasing in $n$.  If $\sum_{j=1}^p (X_{(n+1) j}-X_{i j})=0$ for all $i=1, \ldots, n$, we can let $\delta=1$.  In practice, we can simply take $\delta=1+a \underset {i''} {\min} \left \{Y_{i''}^2/\sum_{j=1}^p (X_{(n+1) j}-X_{i''j}) \right\}$, where $\sum_{j=1}^p (X_{(n+1) j}-X_{i''j})<0$ and $0<a<1$.\\

For a prediction interval,  its mean length is a measure of efficiency.  Our main goal is to circumvent the aforementioned challenges in the full conformal prediction framework, not to construct a nearly optimally efficient prediction interval,  although the latter can be an important topic for future research.  However, for any confidence level $1-\alpha$,  there is no finite-sample valid $100(1-\alpha)\%$ prediction interval with the shortest mean length, as the next theorem shows.  Therefore,  one shall not aim at such an ``ideal'' prediction interval.  

\begin{thm}
\label{thm:length}
For a given sample size $n$ and any confidence level $0<1-\alpha<1$,  there is no finite-sample valid (distribution-free) $100(1-\alpha)\%$ prediction interval with shortest mean length. 
\end{thm}

\begin{proof}
Suppose the statement of the theorem is false.  Then there will be a finite-valid $100(1-\alpha)\%$ prediction interval  $(L(X_{n+1}, Z^n), U(X_{n+1}, Z^n))$ for some confidence level $0<1-\alpha<1$ such that $0<d_n=\E[U(X_{n+1}, Z^n))-L(X_{n+1}, Z^n))]<\infty$ is the shortest mean length of a finite-sample valid $100(1-\alpha)\%$ prediction interval,  irrespective of the distribution of $Y$.  Let $0<l<r<n+1$ such that $(r-l)/(n+1)\geq 1-\alpha$. 
Then $(Y_{(l)}, Y_{(r)})$ is also a finite-sample valid $100(1-\alpha)\%$ prediction interval.  We can always find  $\sigma>0$ sufficiently small such that 
\[
0<\frac{2(n-1)\sigma}{\sqrt{2n-1}}<d_n.
\]
Now consider the case where $Y$ follows a distribution such that $Var(Y)=\sigma^2$.  By the bounds of $E(Y_{(1)})$ and $E(Y_{(n)})$ (e.g.,  David and Nagara 2003, Section~4.2), we have
\[
0<\E \left[Y_{(r)}-Y_{(l)}\right]\leq \E\left[Y_{(n)}-Y_{(1)}\right]\leq \frac{2(n-1)\sigma}{\sqrt{2n-1}}<d_n,
\]
contradicting the definition of $d_n$.
\end{proof}

\section{Illustration}

Throughout this section, we will use the following notation:
\begin{enumerate}
\item[$\bullet$]$\Binom(k, p)$ denotes the binomial distribution with trial number parameter $k$ and success probability $0<p<1$;
  \item [$\bullet$]$\gam(a, b)$ represents the gamma distribution with shape parameter $a$ and rate parameter $b$, i.e., the gamma distribution whose probability density function is 
              \[
              f(x)=\frac{b^a}{\Gamma(a)}x^{a-1}e^{-b x}, \quad x>0,
              \]
               where $\Gamma(\alpha)=\int_0^\infty t^{\alpha-1} e^{-t} dt$ is the Gamma function;
\item[$\bullet$]$\nm(\mu, \sigma^2)$  stands for the normal distribution with mean $\mu$ and variance $\sigma^2$; 
\item[$\bullet$] $\Pareto(\eta, \beta)$ symbolizes the Type-II Pareto distribution with the probability density function
              \[
              g(x)=\frac{\eta \beta^{\eta}}{(x+\beta)^{\eta+1}}, \quad x>0.
              \]
\end{enumerate}

\subsection*{Example~A}
For $\alpha=0.1$,  we generate $N=3,000$  random samples of size $n=101$ from the following model:
\[
Y=X_1+X_2+\varepsilon,
\]
where $X_1,$ $X_2$, and $\varepsilon$ are independent,  and $X_1\sim \nm(0, 2)$,  $X_2\sim\nm(0, 1)$, and $\varepsilon\sim \nm(0,  0.4)$.  For each sample, the response values of the first $100$ sample points and all the $101$ values of the two features are used to construct the conformal prediction intervals in the above three theorems.  We let $(\delta+\gamma)/1-\eta=1$ and $-1$ in Theorem~\ref{thm:interval_rightbounded} and Theorem~\ref{thm:interval_leftbounded},  respectively.  For the bounded prediction interval in Theorem~\ref{thm:interval_two},  we take $\delta=1+0.8 \underset {i''} {\min} \left \{Y_{i''}^2/\sum_{j=1}^p (X_{(n+1) j}-X_{i''j}) \right\}$, where $\sum_{j=1}^p (X_{(n+1) j}-X_{i''j})<0$. Then, the $101$st response value is treated as the future response value we want to predict.  We estimate the coverage probability of each prediction interval as $K/N$ where $K$ is the number of times it contains the $101$st response value.  For the $(a, b)$-shaped $100(1-\alpha)\%$ prediction interval,  we report its mean length as a measure  of its efficiency.   For the $(-\infty, a)$-shaped prediction interval,  we measure its efficiency using $\E[a]$.  Likewise,  the efficiency of the  $(a, \infty)$-shaped prediction interval  is measured with $\E[a]$.  We do the same for split conformal prediction with least squares absolute test residuals/predictive errors in Lei et al. (2018) and the transformation method with the linear function $h(t_1,t_2)=1+t_1+t_2$ in Hong (2016b).  That is, we let $R_i$ be the $i$-th least squares absolute test predictive error in the split conformal algorithm in Lei et al. (2018) to obtain an equal-tailed bounded prediction interval.  We obtain the two one-sided prediction intervals in the obvious way by using the same split conformal prediction algorithm with $R_i$ being the $i$-th test least squares predictive error.  For the transformation method in Hong (2026b),  we shift $(-\infty, W_{(r_1)})$, $(W_{(r_2)}, +\infty)$ and $(W_{(r_3)}, W_{(n-r_3+1)})$ by $h(X_{n+1, 1}, X_{n+1, 2})$ to obtain the three prediction intervals for $Y_{n+1}$. Table~\ref{table:1} summarizes the results.

\begin{table}[!ht]
\begin{center}
\begin{tabular}{l|ccc}
\hline
Prediction Interval Shape    & Full     & Split  & Transformation\\
\hline
$\quad \quad (-\infty, a)$              & 0.91 (0.82)     & 0.91 (0.81)  &  0.91 (0.82)\\
$\quad \quad (a, \infty)$               & 0.91 (-0.86)   & 0.90 (-0.91) & 0.91 (-0.88)\\
$\quad \quad (a, b)$           	   &  0.92  (6.16)    & 0.90 (2.18)	  & 0.90 (2.14)\\
\hline
\end{tabular}
\end{center}
\caption{Coverage probabilities (efficiency) of three types of the $90\%$ prediction intervals in Example~A, based on full conformal prediction with proposed non-conformity measures,  split conformal prediction with least square absolute test predictive errors,   and the transformation method with a linear transformation.}
\label{table:1}
\end{table}

For all three desired shapes,  prediction intervals based on the proposed method,  split conformal prediction,  and the transformation method all achieve the nominal cover probability.  This is no surprise because these prediction intervals are all provably finite-sample valid.  For the $(-\infty, a)$-shaped prediction interval, the three methods are almost equally efficient.  The proposed method slightly outperforms the other two methods for the  $(a, \infty)$-shaped prediction interval.  However, for the $(a, b)$-shape prediction intervals, the proposed method is much more conservative than the other two methods. Our explanation for this is that when the true data-generating model is a bona fide linear model,  the information needed for constructing a finite-sample valid prediction interval is mainly contained in the order of data. Thus,  the non-conformity measure in Theorem~\ref{thm:interval_two},  a quadratic functional of data, might not efficiently utilize the information in the data to create a finite-sample valid prediction interval.

\subsection*{Example~B}
We perform the same simulation with the same values of $\alpha$, $N$ and $n$ as in Example~A,  except that (a)~data are generated from the following model:
\[
Y=X_1^2-X_2^3+X_3+\varepsilon,
\]
where $X_1, X_2$, and $\varepsilon$ are independent,  and $X_1\sim \nm(0, 4)$, $X_2\sim \gam(2, 1)$, $X_3\sim \Binom(2, 0.4)$, and $\varepsilon\sim \Pareto(2, 7)$, and (b)~we use $h(t_1, t_2, t_3)=1+t_1+t_2+t_3$ in the transformation method. The results are summarized in Table~\ref{table:2}.

\begin{table}[!ht]
\begin{center}
\begin{tabular}{l|ccc}
\hline
Prediction Interval Shape    & Full     & Split  & Transformation\\
\hline
$\quad \quad (-\infty, a)$              & 0.90 (9.14)           & 0.91 (10.11)      &  0.90 (9.14)\\
$\quad \quad (a, \infty)$               & 0.90 (-61.19)         & 0.90 (-48.21)  & 0.91 (-68.12)\\
$\quad \quad (a, b)$           	   &  0.90  (118.41)    & 0.91 (99.30)	    & 0.90 (132.29)\\
\hline
\end{tabular}
\end{center}
\caption{Coverage probabilities (efficiency) of three types of the $90\%$ prediction intervals in Example~B, based on full conformal prediction with proposed non-conformity measures,  split conformal prediction with least square absolute test predictive errors,   and the transformation method with a linear transformation.}
\label{table:2}
\end{table}

This time the true data-generating model is not a classical linear model, since the response variable $Y$ is a cubic polynomial of the three predictors and $\E[\varepsilon]\neq 0$.  (Note that none of the three methods needs the ``convenient'' assumption $\E[\varepsilon]=0$.) Similar to what we have seen in Example~A,  all prediction intervals attain the nominal coverage probability.  For the $(-\infty, a)$-shaped prediction intervals,  the proposed method and the transformation method are tied, and they both outperform split conformal prediction.  For the $(a, \infty)$-shaped prediction intervals as well as the $(a, b)$-shaped prediction intervals,  split conformal prediction is more efficient than the other two methods, with the proposed method being slightly more efficient than the transformation method. 

\subsection*{Example~C}
We repeat the same simulation with the same values of $\alpha$, $N$ and $n$ as in Example~B,  except that the true data-generating model is
\[
Y=X_1-X_2+e^{X_3}-\varepsilon,
\]
where $X_1, X_2$, and $\varepsilon$ are independent,  and $X_1\sim \Pareto(3, 2)$, $X_2\sim \nm(0, 1)$,  $X_3\sim \gam(2, 2)$, and $\varepsilon\sim \Binom(10, 0.4)$.  
The results are summarized in Table~\ref{table:2}.

\begin{table}[!ht]
\begin{center}
\begin{tabular}{l|ccc}
\hline
Prediction Interval Shape  & Full     & Split  & Transformation\\
\hline
$\quad \quad (-\infty, a)$              & 0.90 (7.22)           & 0.90 (7.47)    &  0.90 (7.31)\\
$\quad \quad (a, \infty)$                & 0.90 (-7.59)         & 0.90 (-7.69)   & 0.92 (-7.83)\\
$\quad \quad (a, b)$           	   & 0.91 (18.17)    	& 0.91 (21.78)	  & 0.91 (19.84)\\
\hline
\end{tabular}
\end{center}
\caption{Coverage probabilities (efficiency) of three types of the $90\%$ prediction intervals in Example~C, based on full conformal prediction with proposed non-conformity measures,  split conformal prediction with least square absolute test predictive errors,   and the transformation method with a linear transformation.}
\label{table:3}
\end{table}

All prediction intervals in Table~\ref{table:3} achieve their nominal coverage probability, as anticipated. In terms of efficiency, the proposed method outperforms split conformal prediction and the transformation method for all three interval shapes.  This time split conformal prediction is the least efficient one among the $(a, b)$-shaped prediction interval. \\

If a data scientist only cares about finite-sample validity, then any of these three methods will do a good job.  When it comes to efficiency,  we are not prepared to assign supremacy to any one of them over the other two.  However,  the above three examples seem to suggest that split conformal prediction is likely to be an excellent choice when the response variable is truly linear in predictors or a polynomial of predictors,  and that the proposed method and the transformation can work well when the true data-generating mechanism is complicated.  

\section{Concluding remarks}

Within the full conformal prediction framework,  a conformal prediction region is guaranteed to be finite-sample valid.  However,  the determination of a conformal prediction region generally requires a data scientist to evaluate the plausibility function at infinitely many values, which cannot be completed in practice.  If we evaluate the plausibility function for only a grid of possible values,  the resulting region,  an approximation of the conformal prediction region,  is no longer guaranteed to be finite-sample valid, let alone there are two other challenges: (i) the computation required can still be prohibitively expensive,  and (ii) the resulting prediction region might not be of a desired shape.  Some colleagues tend to believe that a monotonic nonconformity measure will overcome these challenges.  Our investigation showed that this is false, among other insights into the relationship between the monotonicity of the non-conformity measures, the monotonicity of the plausibility function, the shape of the conformal prediction region, and the exact determination of the conformal prediction regions.  Our investigation also showed that it is challenging to give a hard-and-fast rule for choosing a non-conformity measure so that this issue can be avoided and the resulting conformal prediction region can be of a desired shape.  

Driven by the principle of parsimony, we propose a non-conformity measure based on a multivariate quadratic polynomial.  With the proposed non-conformity measure, we can not only avoid three common practical challenges in conformal prediction but also generate conformal prediction intervals of three common shapes.  Simple and easy to implement, our proposal enables a data scientist to both determine the prediction interval exactly and control its shape in the full conformal prediction framework.

\section*{Acknowledgments}
We thank the two anonymous reviewers for many helpful comments and suggestions.

\section*{Conflict of interest}
The authors declare no conflict of interest. 

\section{Appendix}
This appendix briefly documents the key results in unsupervised learning.  

\subsection{Notation and setup}

Suppose our observations  consist of a sequence of exchangeable random variables $X_1, X_2, \ldots$, where each $X_i$ follows a distribution $\prob$.  We want to construct a prediction interval of the next observation $X_{n+1}$, based on past observations of $X^n = \{X_1,\ldots,X_n\}$.  The conformal prediction algorithm in this case is the following Algorithm~\ref{algo:conformal}:

\begin{algorithm}[H]
\label{algo:conformal}
Initialize: data $x^n = \{x_1,\ldots,x_n\}$ , non-conformity measure $M$\;
\For{each possible $x$ value} {
Set $x_{n+1} = x$ and write $x^{n+1} = x^n \cup \{x_{n+1} \}$\;
Define $\mu_i = M(x^{n+1} \setminus \{x_i\}, x_i)$ for $i=1,\ldots,n,n+1$\;
Compute $\pl_{x^n}(x) = (n+1)^{-1} \sum_{i=1}^{n+1} 1\{\mu_i \geq \mu_{n+1}\}$\;
}
Return $\pl_{x^n}(x)$ for each $x$\;
\caption{\bf Conformal prediction (unsupervised)}
\end{algorithm}

A $100(1-\alpha)\%$ conformal prediction region can be constructed as follows: 
\begin{equation}
\label{eq:region_iid}
C_\alpha(X^n) = \{x: \pl_{X^n}(x) > t_n(\alpha)\},
\end{equation}
where $0<\alpha<1$ and $t_n(\alpha)=(n+1)^{-1}\lfloor (n+1)(1-\alpha)\rfloor$.  The finite-sample validity still holds:
\begin{equation*}
\label{eq:jointvalidity_iid}
\prob^{n+1}\{X_{n+1}\in C_{\alpha}(X^n) \}\geq 1-\alpha\quad \text{for all $(n, \prob)$},
\end{equation*}
where $\prob^{n+1}$ is the joint distribution for $(X_1, \ldots, X_n, X_{n+1})$.  

Let $B=\{x_1, \ldots, x_n\}$ be a bag of  observations of size $n$. For $i=1, \ldots, n$, let $(x_1 \ldots, x_n)$ be the observations in $B$.  Suppose $x$ is a provisional value of a future observation of $X_{n+1}$ to be predicted.  In this case,  a non-conformity measure $M$ is said to be \emph{monotonically increasing} if 
\[
x\leq x' \Longrightarrow M(B, x)\leq M(B, x'); 
\]
$M$ is said to be \emph{monotonically decreasing} if 
\[
x\leq x' \Longrightarrow M(B, x)\geq M(B, x'). 
\]
$M$ is said to be \emph{monotonic} if it is either monotonically increasing or monotonically decreasing.

In unsupervised learning,   (\ref{eq:property*}) corresponds to 
\begin{equation}
\label{eq:property'}
\mu_i\geq \mu_{n+1} \text{ if and only if $X_i\geq X_{n+1}$ or $X_i\leq X_{n+1}$}.
\end{equation}

Corresponding to Questions~(I)---(IX), we have the following Questions~(I')---(IX').

\begin{enumerate}
\item[(I')]Does monotonicity of the non-conformity measure $M$ imply (\ref{eq:property'})?
\item[(II')]Does (\ref{eq:property'}) imply the monotonicity of $M$? 
\item[(III')]Is monotonicity of $M$ a necessary condition for $C_{\alpha}(X^n)$ to be an interval?
\item[(IV')]Does the monotonicity of $M$ imply that the resulting conformal prediction region is an interval?
\item[(V')]Is the monotonicity of $M$ a necessary condition for the exact determination of $C_{\alpha}(X^n)$?
\item[(VI')]Is (\ref{eq:property'}) a necessary condition for $C_{\alpha}(X^n)$ to be an interval? (Note that the converse is true: (\ref{eq:property'}) implies $C_{\alpha}(X^n)$ is a one-sided interval).
\item[(VII')]Is monotonicity of the plausibility function $\pl_{x^n}$ a necessary condition for $C_{\alpha}(X^n)$ to be a one-sided interval? (Note that the converse holds, i.e.,  monotonicity of the plausibility function implies $C_{\alpha}(X^n)$ is a one-sided interval.)
\item[(VIII')]Is (\ref{eq:property'}) a necessary condition for exact determination of $C_{\alpha}(X^n)$? (Note that the converse is true: (\ref{eq:property'}) implies $C_{\alpha}(X^n)$ is a one-sided interval; hence it implies the exact determination of $C_{\alpha}(X^n)$.)
\item[(IX')]Is monotonicity of the plausibility function $\pl_{x^n}$ a necessary condition for the exact determination of $C_{\alpha}(X^n)$? (Note that the converse holds, i.e.,  monotonicity of the plausibility function implies $C_{\alpha}(X^n)$ is a one-sided interval; hence, it implies the exact determination of $C_{\alpha}(X^n)$.)

\end{enumerate}

Answers to Questions~(I')---(IX') follow from answers to Questions~(I)---(IX) by ignoring the predictors and treating the response variable as the random variable of interest in unsupervised learning.  For example,  if we take  $M(B, x)=\min\{x_1, \ldots, x_n\}+x$, we will obtain an example to show the answer to Question (I') is negative.  For the sake of brevity, we give answers to Questions~(I')---(IX') without proofs. \\

\noindent \textbf{Answer to Question (I'): No.  Monotonicity of $M$ need not imply (\ref{eq:property'}).}

\bigskip

\noindent \textbf{Answer to Question (II'): No.  (\ref{eq:property'}) does not imply $M$ is monotonic.}

\bigskip

\noindent \textbf{Answer to Question (III'): No.  Monotonicity of $M$ is not a necessary condition for $C_{\alpha}(X^n)$ to be an interval?}

\bigskip

\noindent \textbf{Answer to Question (IV'): No. Monotonicity of $M$ need not imply $C_{\alpha}(X^n)$ is an interval.}

\bigskip

\noindent \textbf{Answer to Question (V'): No.  Monotonicity of $M$ is not necessary condition for exact determination of $C_{\alpha}(X^n)$.}

\bigskip

\noindent \textbf{Answer to Question (VI'): No.  (\ref{eq:property'}) is not a necessary condition for $C_{\alpha}(X^n)$ to be an interval.}

\bigskip

\noindent \textbf{Answer to Question (VII'): Yes.  Monotonicity of the plausibility function $\pl_{x^n}$ a necessary condition for $C_{\alpha}(X^n)$ to be a one-sided interval.}

\bigskip

\noindent \textbf{Answer to Question (VIII'): No.  (\ref{eq:property'}) is not a necessary condition for exact determination of $C_{\alpha}(X^n)$.}

\bigskip

\noindent \textbf{Answer to Question (IX'): No.  Monotonicity of the plausibility function $\pl_{x^n}$ is not a necessary condition for exact determination of $C_{\alpha}(X^n)$.} \\

\begin{thm}
\label{thm:predproperty2}
If $\lfloor (n+1)\alpha \rfloor< 1$, then the $(1-\alpha)\%$ prediction region $C_{\alpha}(X^n)$ given by (\ref{eq:region_iid}) is $\mathbb{R}$.
\end{thm}

\begin{proof}
Completely similar to the proof of Theorem~2.
\end{proof}

Now consider the non-conformity measure
\begin{equation}
\label{eq:ncmeasure2}
M(B, x)=\lambda x^2+\theta x+\kappa \sum_{j=1}^n x_j,
\end{equation}
where  $B=\{x_1, \ldots, x_n\}$ and $\lambda, \theta$, and $\kappa$ are constants to be decided by the user.  Following the same line of reasoning as in the previous section,  we can see that the following three theorems hold.  Note that Theorems~\ref{thm:interval_right2} and \ref{thm:interval_left2} recover the traditional non-parametric one-sided prediction  intervals based on order statistics (e.g., Wiks 941; Fligner and Wolfe 1976; Frey 2013).

\begin{thm}
\label{thm:interval_right2}
Suppose $0<\alpha<1$, $\lfloor (n+1)\alpha \rfloor\geq 1$,  and the non-conformity measure is given by (\ref{eq:ncmeasure2}).  If $\lambda=0$, $\theta=1$,  $\kappa = -1$, then the $100(1-\alpha)\%$ conformal prediction region $C_{\alpha}(X_n)$ is the one-sided interval $(-\infty,  X_{(r_1)})$.
\end{thm}

\begin{thm}
\label{thm:interval_left2}
Suppose $0<\alpha<1$, $\lfloor (n+1)\alpha \rfloor\geq 1$,  and the non-conformity measure is given by (\ref{eq:ncmeasure2}).  If $\lambda=0$, $\theta=-1$,  $\kappa = 1$, then the $100(1-\alpha)\%$ conformal prediction region $C_{\alpha}(X_n)$ is the one-sided interval $(X_{(r_2)}, \infty)$.
\end{thm}

\begin{thm}
\label{thm:interval_two2}
Suppose $0<\alpha<1$, $\lfloor (n+1)\alpha \rfloor\geq 1$,  and the non-conformity measure is given by (\ref{eq:ncmeasure2}).  
If $\lambda=1$,  $\theta=0$, $\kappa \neq 0$, then the $100(1-\alpha)\%$ conformal prediction region $C_{\alpha}(X_n)$ is the bounded interval $(a_{(r_3)}, b_{(n-r_3+1)})$, where $a_i = \min\{-(\kappa +X_i), X_i\}$ and $b_i = \max\{-(\kappa+X_i), X_i\}$ for $1\leq i\leq n$.
\end{thm}

\section*{References}
\begin{description}

\item{} Barber, F.R., Candes, E.~J., Ramdas, A., and Tibshirani, R.~J.~(2021). The limits of distribution-free conditional predictive inference.  \emph{Information and Inference: A Journal of the IMA}~10, 455--482.

\item{} Bates, S., Candes, E., Lei, Li, Romano, Y.~and Sesia, M.~(2023). Testing for outliers with conformal $p$-values. \emph{Annals of Statistics}~51(1), 149--178,

\item{} Berk, R. Brown, L., Buja, A. Zhuang, K. and Zhao, L.~(2013). Valid post-selection inference, \emph{Annals of Statistics}, 41, 802--837. 

\item{} Burnaev, E.~and Vovk, V.~(2014). Efficiency of conformalized ridge regression.  \emph{Journal of Machine Learning Research}~35, 1--18.

\item{} Claeskens, G.~and Hjort, N.L.~(2008).  \emph{Model Selection and Model Averaging}.  Cambridge University Press: Cambridge, UK. 

\item{} David, H.A.~and Nagaraja, H.N.~(2003).  \emph{Order Statistics}, Third Edition.  John Wiley $\&$ Sons,  Inc.: Hoboken, New Jersey.

\item Fligner, M.~and Wolfe, D.A.~(1976). Some applications of sample analogues to the probability integral transformation and a coverage property.  \emph{The American Statistician}~30(2), 78--85.

\item{} Frey, J.~(2013). Data-driven non-parametric prediction intervals. \emph{Journal of Statistical Planning and Inference}~143, 1039--1048.

\item{} Hong, L., Kuffner, T. ~and Martin, R. ~(2018). On overfitting and post-selection uncertainty assessments.  \emph{Biometrika}~105(1), 221--224.

\item{} Hong, L.~and Martin, R.~(2021). Valid model-free prediction of future insurance claims. \emph{North American Actuarial Journal}~25(4), 473--483.

\item{} Hong, L.~(2026a).  Conformal prediction of future insurance claims in the regression problem, \emph{European Actuarial Journal}, to appear,  \url{https://link.springer.com/article/10.1007/s13385-026-00445-y}.

\item{} Hong, L.~(2026b).  A new strategy for finite-sample valid prediction of future insurance claims in the regression setting. \url{https://arxiv.org/abs/2601.21153}.

\item{} Kuchibhotla, A.K. Kolassa, J.E.~and Kuffner, T.A.~(2022). Post-Selection Inference. \emph{Annual Review of Statistics and Its Application}~9,  505--527.

\item{} Leeb, H. (2009). Conditional predictive inference post model selection, \emph{Annals of Statistics}, 37, 2838--2876.

\item{} Lei, J., Robins, J.~and Wasserman, L.~(2013). Distribution-free prediction sets. \emph{Journal of American Statistical Association}~108(501), 278--287.

\item{} Lei, J.~and Wasserman, L.~(2014). Distribution-free prediction bands for non-parametric regression.  \emph{Journal of Royal Statistical Society-Series~B}~(76), 71--96.

\item{} Lei, J.,  G'Sell, M., Rinaldo, A., Tibshirani, R.J.~and Wasserman, L.~(2018). Distribution-free predictive inference for regression.  \emph{Journal of the American Statistical Association}~113(523), 1094--1111.

\item{} Lei, J.~(2019). Fast exact conformalization of the lasso using piecewise linear homotopy. \emph{Biometrika}~106(4), 749--764.

\item{} Li, D.~(2024). Generalized fast exact conformalization. \emph{38th Conference on Neural Information Processing System}~(NeurIPS 2024).

\item{} Martin, R.~and Lingham, R.T.~(2016). Prior-free probabilistic prediction of future observations. \emph{Technometrics}~58(2), 226--235.

\item{} Papadopoulus H., Proedrou, K.,  Vovk, V.~and Gammerman, A.~(2002). Inductive confidence machine learning for regression. \emph{Machine Learning ECML 2002}, eds T. Elomaa, H. Hannila, and H.Toiovenen. Springer: New York, 345--356.

\item{} Romano, Y., Patterson, E.~and Candes, E.J.~(2019). Conformalized Quantile Regression. \emph{33rd Conference on Neural Information Processing System}~(NeurIPS 2019).

\item{} Shafer, G.~and Vovk, V.~(2008). A tutorial on conformal prediction. \emph{Journal of Machine Learning}~9, 371--421.

\item{} Shmueli, G.~(2010). To explain or predict? \emph{Statistical Sciences}~25(3), 289--310. 

\item{} Vovk, V., Gammerman, A., and Shafer, G.~(2005). \emph{Algorithmic Learning in a Random World}. New York: Springer.

\item{} Vovk, V., Nouretdinov, I.~and Gammerman, A.~(2009). On-line predictive linear regressioin. \emph{Annals of Statistics}~37, 1566--1590.

\item{} Vovk, V., Shen, J., Manokhim, V.~and Xie, M.~(2019). Nonparametric predictive distributions based on conformal prediction. \emph{Machine Learning}~108:445---474.

\item{} Wilks, W.~(1941).  Determination of sample sizes for setting tolerance limits. \emph{Annals of Mathematical Statistics}~12, 91--96.
\end{description}

\end{document}